\documentclass[cis, authoryear]{ipart}

\usepackage[T1]{fontenc}
\usepackage[utf8]{inputenc}

\usepackage{microtype}

\usepackage{booktabs}
\usepackage{multirow}

\usepackage{xcolor}

\RequirePackage{hyperref}
\hypersetup{colorlinks=false, hidelinks}

\usepackage[lined,linesnumbered,ruled,vlined,algo2e]{algorithm2e}
\SetKw{KwForIn}{in}
\SetKw{KwUntil}{until}
\SetKwInput{KwInputs}{Inputs}
\SetKwInput{KwOutputs}{Outputs}

\startlocaldefs
\theoremstyle{plain}

\newtheorem{remark}{Remark}[section]
\newtheorem{proposition}{Proposition}[section]
\endlocaldefs

\usepackage{amsmath, amssymb, bm, mathtools}

\newcommand{\cu}[1]{
	\ifcat\noexpand#1\relax
	\bm{#1}
	\else
	\mathbf{#1}
	\fi
}

\newcommand{\diff}{\mathop{}\!\mathrm{d}}

\newcommand{\cond}{{\;|\;}}

\let\lim\relax
\DeclareMathOperator*{\lim}{lim\,}  

\let\grad\relax
\DeclareMathOperator{\grad}{\nabla\!}

\newcommand{\expecsym}{\operatorname{\mathbb{E}}}     
\newcommand{\covsym}{\operatorname{Cov}}     
\newcommand{\varrsym}{\operatorname{Var}}     
\newcommand{\diagsym}{\operatorname{diag}}     
\newcommand{\tracesym}{\operatorname{tr}}           

\let\expec\relax
\let\cov\relax
\let\varr\relax
\let\diag\relax
\let\trace\relax

\makeatletter
\newcommand{\expec}{\@ifstar{\@expecauto}{\@expecnoauto}}
\newcommand{\@expecauto}[1]{\expecsym \left[ #1 \right]}
\newcommand{\@expecnoauto}[1]{\expecsym [#1]}

\newcommand{\cov}{\@ifstar{\@covauto}{\@covnoauto}}
\newcommand{\@covauto}[1]{\covsym \left[ #1 \right]}
\newcommand{\@covnoauto}[1]{\covsym [#1]}

\newcommand{\varr}{\@ifstar{\@varrauto}{\@varrnoauto}}
\newcommand{\@varrauto}[1]{\varrsym \left[ #1 \right]}
\newcommand{\@varrnoauto}[1]{\varrsym [#1]}

\newcommand{\diag}{\@ifstar{\@diagauto}{\@diagnoauto}}
\newcommand{\@diagauto}[1]{\diagsym \left( #1 \right)}
\newcommand{\@diagnoauto}[1]{\diagsym (#1)}
\newcommand{\diagbig}[1]{\diagsym \bigl( #1 \bigr)}

\newcommand{\trace}{\@ifstar{\@traceauto}{\@tracenoauto}}
\newcommand{\@traceauto}[1]{\tracesym \left( #1 \right)}
\newcommand{\@tracenoauto}[1]{\tracesym (#1)}

\makeatother

\newcommand*{\trans}{{\mkern-1.5mu\mathsf{T}}}

\newcommand*{\R}{\mathbb{R}} 

\makeatletter
\@ifundefined{thmenvcounter}{}
{%
	\newtheorem{envcounter}{EnvcounterDummy}[\thmenvcounter]
	
	\newtheorem{proposition}[envcounter]{Proposition}

	\newtheorem{remark}[envcounter]{Remark}

}
\makeatother

\newcommand{\refmeasure}{\pi_{\mathrm{ref}}}

\pubyear{2022}
\volume{0}
\issue{0}
\firstpage{1}
\lastpage{9}
\arxiv{2506.01083}

\begin{document}

\begin{frontmatter}
\title[Generative diffusion posterior sampling]{Generative diffusion posterior sampling for informative likelihoods\protect\thanksref{T1}}
\thankstext{T1}{This work was partially supported by 1) the Wallenberg AI, Autonomous Systems and Software Program (WASP) funded by the Knut and Alice Wallenberg Foundation, and 2) the Stiftelsen G.S Magnusons fond (MG2024-0035). }

\begin{aug}
    \author{\fnms{Zheng} \snm{Zhao}\ead[label=e1]{zheng.zhao@liu.se}}
    \address{Division of Statistics and Machine Learning\\
    		Link\"{o}ping University, Sweden\\
             \printead{e1}}
\end{aug}
\received{\sday{3} \smonth{1} \syear{2022}}

\begin{abstract}
Sequential Monte Carlo (SMC) methods have recently shown successful results for conditional sampling of generative diffusion models. 
In this paper we propose a new diffusion posterior SMC sampler achieving improved statistical efficiencies, particularly under outlier conditions or highly informative likelihoods. 
The key idea is to construct an observation path that correlates with the diffusion model and to design the sampler to leverage this correlation for more efficient sampling. 
Empirical results conclude the efficiency.
\end{abstract}



\end{frontmatter}

\section{Problem formulation}
\label{sec:intro}
Consider a probability distribution $\pi(\cdot)$ extended as a time marginal of
\begin{equation}
	\begin{split}
		q_{0:N}(u_{0:N}) &= q_0(u_0)\prod_{k=1}^N q_{k\cond k-1}(u_k \cond u_{k-1}), \quad q_0 = \refmeasure, \quad q_N = \pi, \\
		p_{0:N}(x_{0:N}) &= p_0(x_0)\prod_{k=1}^N p_{k\cond k-1}(x_k \cond x_{k-1}), \quad p_0 = \pi, \quad p_N = \refmeasure, 
	\end{split}
	\label{equ:generative-model}
\end{equation}
where $p_{k\cond k-1}$ and $q_{k \cond k-1}$ stand for the noising and denoising transition distributions on $\R^d$, respectively. 
This pair of processes, when exists, constitutes a generative diffusion model, allowing efficient sampling from $\pi$ by first sampling from a reference distribution $\refmeasure$ and then sequentially denoising via $q_{k \cond k-1}$ for $k=1,2,\ldots, N$. 
The construction of such models has been a central topic in recent research, focusing on accelerating the computation and reducing the statistical approximation errors.
A common approach is to choose the noising $p_{k\cond k-1}$ as an Ornstein--Uhlenbeck process, so that approximately $p_N \approx \mathrm{N}(0, I)$ for large enough $N$, and that the denoising process $q_{k \cond k-1}$ can be estimated with score matching~\citep{Song2021scorebased, Ho2020DDPM}, or more recently, scoring rules~\citep{Shen2025reverse}. 
Another useful construction is the dynamic Schrödinger bridge, which constructs both noising and denoising processes by minimising a transportation cost relative to a reference process~\citep{DeBortoli2021diffusion}. 
For a detailed exposition of generative diffusion models, we refer the readers to, for example, \citet{Zhao2024rsta, Albergo2023stochastic, Benton2024}. 
Throughout this paper, we assume that such a generative diffusion model in Equation~\eqref{equ:generative-model} is given. 

The overall goal of the paper is to sample from the conditional/posterior distribution $\pi(\cdot \cond y) \propto f(y \cond \cdot) \, \pi(\cdot)$, where $f$ is a given likelihood function, and $y\in\R^c$ denotes an observation~\citep{Zhao2024rsta}. 
This task is at the core interests of Bayesian statistics in general, and in the context of generative diffusion models specifically, this has enabled many downstream applications, such as image restoration problems~\citep{Luo2024rsta}.
Recently, a few methods have been proposed to address this problem, notably the diffusion posterior sampling~\citep{Chung2023DPS} and their variants~\citep[see survey by][]{Daras2024survey}. 
These methods rely on constructing a conditional diffusion process characterised by transition $\widetilde{q}_{k\cond k-1}(u_k \cond u_{k-1}, y)$ that encodes the likelihood $f$ and observation, to target $\pi(\cdot \cond y)$. 
However, these methods are intrinsically biased due to necessary approximations to $\widetilde{q}_{k\cond k-1}$. 
This has then motivated another class of methods based on Feynman--Kac models and sequential Monte Carlo~\citep[SMC,][]{ChopinBook2020}, to remove such biases~\citep[e.g.,][]{Wu2023practical, Cardoso2024monte, Janati2024DC, Kelvinius2025DDSMC, Corenflos2024FBS}. 
At the heart, they work by using $\widetilde{q}$ as a proposal, and then correct the proposal samples with importance weights. 
While this type of methods has shown to be empirically working well, their statistical efficiencies (e.g., effective sample size) can be significantly degraded by high dimension $d$ and informative likelihood. 
This challenge has been recognised in the SMC literature, such as the seminal work by~\citet{DelMoral2015}.
However, in the context of SMC conditional sampling of generative diffusion models, this is under investigated. 
As such, the specific purpose of this paper is to address the problem when the likelihood being highly informative or when the observation is an outlier. 

Our contributions are as follows. 
We develop a new Feynman--Kac and SMC-based framework to sample conditional/posterior distributions with generative diffusion model prior. 
The method is designed to stay statistically robust, even when the observation is an outlier. 
We show via high-dimensional experiments that the proposed method outperforms the state of the art. 

The paper is organised as follows. 
In Section~\ref{sec:fk} we introduce necessary premises and explain how the conditional sampling works with a Feynman--Kac model and SMC sampler. 
Then, in Section~\ref{sec:corr-fk} we show our construction of the Feynman--Kac model to facilitate sampling under outlier observations, followed by experiments in Section~\ref{sec:experiments}. 

\section{Generative Feynman--Kac models}
\label{sec:fk}
A Feynman--Kac model $Q_{0:N}$ defines a probability distribution over a temporal sequence of $N+1$ random variables:
\begin{equation}
	Q_{0:N}(u_{0:N}) = \frac{1}{Z_N} \, M_0(u_0) \, G_0(u_0) \prod_{k=1}^N M_{k \cond k-1}(u_k \cond u_{k-1}) \, G_k(u_k, u_{k-1}), 
	\label{equ:fk}
\end{equation}
where $M_0$ is the initial distribution, $Z_N$ is the normalising constant, and $M_{k \cond k-1}$ and $G_k$ are Markov (proposal) transition distribution and potential function, respectively. 
Given the model components, a natural approach for sampling $Q_{0:N}$ is sequential Monte Carlo (SMC) which exploits the model's temporal structure, see Algorithm~\ref{alg:smc} for implementation. 
At its heart, this can be viewed as a tempering sequence, gradually evolving from $Q_0$ to $Q_N$. 
For detailed expositions of Feynman--Kac models and SMC samplers, we refer the readers to \citet{ChopinBook2020} and~\citet{DelMoral2004}.

\begin{algorithm2e}[h]
	\SetAlgoLined
	\DontPrintSemicolon
	\KwInputs{Model components $M_0$, $G_0$, $\lbrace M_{k \cond k-1}\rbrace_{k=1}^N$, $\lbrace G_k \rbrace_{k=1}^N$, number of samples $J$.}
	\KwOutputs{Weighted samples of $Q_{0:N}$. }
	Draw $J$ samples $\lbrace U_0^j \rbrace_{j=1}^J \sim M_0$.\;
	Weight $w_0^j = G_0(U_0^j) \, / \, \sum_{i=1}^J G_0(U_0^i)$ \tcp*{For $j=1,\ldots, J$}
	\For{$k=1,2,\ldots, N$}{%
		Resample $\bigl\lbrace (w_{k-1}, U_{k-1}^j) \bigr\rbrace_{j=1}^J$ if needed. \;
		Draw $U_k^j \sim M_{k \cond k-1}(\cdot \cond U_{k-1}^j)$. \tcp*{For $j=1,\ldots, J$}
		Weight $\overline{w}_k^j = w_{k-1}^j \, G_k(U_k^j, U_{k-1}^j)$. \tcp*{For $j=1,\ldots, J$}
		Normalise $w_k^j = \overline{w}_k^j \, / \, \sum_{i=1}^J \overline{w}_k^i$. \tcp*{For $j=1,\ldots, J$}
	}
	\caption{Sequential Monte Carlo (SMC) for sampling Feynman--Kac model $Q_{0:N}$. }
	\label{alg:smc}
\end{algorithm2e}

The Feynman--Kac model $Q_{0:N}$ is particularly useful for generative diffusion models. 
It not only generalises both conditional and unconditional diffusion models but also allows for efficient sampling via the SMC sampler in Algorithm~\ref{alg:smc}. 
As an example, if we set $M_0 = q_0$, $G_0 \equiv 1$, $M_{k \cond k-1} = q_{k \cond k-1}$, and $G_k \equiv 1$, then $Q_{0:N}$ exactly recovers the unconditional denoising path distribution $q_{0:N}$. 
To generalise $Q_{0:N}$ for the conditional sampling, one can \emph{design} a sequence of twisting functions $\lbrace l_k^y \rbrace_{k=0}^N$ that incorporate the likelihood and observation, and construct $Q_{0:N}$ by~\citep[see,][]{Zhao2024rsta, Janati2024DC, Wu2023practical}:
\begin{equation}
	\begin{split}
		&M_0 = q_0, \quad G_0 = l_0^y, \quad l_N^y(\cdot) = f(y \cond \cdot), \\
		&M_{k \cond k-1}(u_k \cond u_{k-1}) \, G_k(u_k, u_{k-1}) = \frac{l_k^y(u_k) \, q_{k \cond k-1}(u_k \cond u_{k-1})}{l_{k-1}^y(u_{k-1})},
	\end{split}
	\label{equ:twisted}
\end{equation}
so that now marginally $Q_N(u_N) \propto l_N^y(u_N) \, q_N(u_N) = f(y \cond u_N) \, \pi(u_N)$ recovers the target posterior distribution. 
Clearly, there is a degree of freedom in choosing the twisting sequence $\lbrace l_k^y \rbrace_{k=0}^N$ under only the terminal constraint $l_N^y(\cdot) = f(y \cond \cdot)$. 
A trivial choice is to use $l^y_k\equiv1$ for $k=0, 1, \ldots, N - 1$, but then the resulting model collapses a single-step importance sampling at $N$, making Algorithm~\ref{alg:smc} useless for high-dimensional diffusion models. 
Therefore, the key challenge here is the design of the twisting functions to facilitate efficient sampling of the SMC sampler in Algorithm~\ref{alg:smc}. 
Ideally, we would like the tempering sequence
\begin{equation}
	Q_k(u_k) \propto l^y_k(u_k) \, q_k(u_k)
\end{equation}
to be ``equidistant'' enough over $k=0, 1, \ldots, N$.

\begin{remark}[Bootstrap and guided construction]
	\label{remark:bootstrap-guided}
	Running Algorithm~\ref{alg:smc} requires the Markov proposal $M_{k \cond k-1}$ and potential function $G_k$ which are implicitly defined through a product structure in Equation~\eqref{equ:twisted}. 
	In practice, the bootstrap
	\begin{equation}
		M_{k \cond k-1} = q_{k \cond k-1}(u_k \cond u_{k-1}), \quad G_k(u_k, u_{k-1}) = \frac{l_k^y(u_k)}{l_{k-1}^y(u_{k-1})}, 
		\label{equ:bootstrap}
	\end{equation}
	and the guided 
	\begin{equation}
		\begin{split}
			M_{k \cond k-1} &\propto l^y_k(u_k) \, q_{k \cond k-1}(u_k \cond u_{k-1}),\\ 
			G_k(u_k, u_{k-1}) &= \frac{l_k^y(u_k) \, q_{k \cond k-1}(u_k \cond u_{k-1})}{l_{k-1}^y(u_{k-1}) \, M_{k \cond k-1}(u_k \cond u_{k-1})}, 
		\end{split}
		\label{equ:guided}
	\end{equation}
	are the most common two constructions of the model~\citep{ChopinBook2020}. 
	While the bootstrap version is easier to implement and is often computationally cheaper, the guided version is statistically superior, particularly for high-dimensional problems. 
	The guided proposal $M_{k\cond k-1}$ is locally optimal, minimising the marginal variance of the weights in Algorithm~\ref{alg:smc}.
\end{remark}

A canonical design, perhaps the most used, of the twisting function~\citep{Wu2023practical} is given by
\begin{equation}
	l^y_k(u_k) = \int f(y \cond u_N) \, q_{N \cond k}(u_N \cond u_k) \diff u_N, 
	\label{equ:twisting-canonical}
\end{equation}
such that the marginal distribution $Q_k(u_k) = \pi(u_k \cond y)$. 
Notably, The resulting Feynman--Kac model gives an Eulerian representation of the conditional score stochastic differential equation~\citep[SDE,][Eq. 14]{Song2021scorebased} at discrete times when the denoising model $q_{k \cond k-1}$ is given by a score SDE. 
Precisely, if the noising process is an SDE $\diff X_t = a(X_t, t) \diff t + \diff B_t$ with $X_0 \sim \pi$, then $Q_{0:N}$ is the finite-dimensional distribution of the conditional SDE
\begin{equation}
	\begin{split}
		\diff U_t &= -a(U_t, T - t) + \grad \log p_{T - t}(U_t) + \grad \log l^y_{T - t}(U_t)  \diff t + \diff W_t, \\
		U_t &\sim \refmeasure(\cdot \cond y),
	\end{split}
	\label{equ:cond-sde}
\end{equation}
at discrete times $t_0, t_1, \ldots, t_N$. 
The score $\grad \log l^y_k(\cdot)$, sometimes referred to as a guidance or control term, is added alongside with the unconditional score function $\grad \log p_k(\cdot)$ in the conditional SDE. 
Although the exact twisting function in Equation~\eqref{equ:twisting-canonical} is intractable in practice, many established approximations have been proposed~\citep[see surveys by][]{Daras2024survey, Luo2024rsta}.
Furthermore, these approximations can be directly integrated in the Feynman--Kac model, for improved statistical performance~\citep[e.g.,][]{Kelvinius2025DDSMC} with SMC samplers. 
These are the reasons why this canonical twisting construction is popularly used. 
As an example,~\citet{Wu2023practical} combine Tweedie's formula and first-order approximation to the integral: $l^y_k(u_k) \approx f(y \cond \expec{X_0 \cond X_{N-k} = u_k})$ proposed by~\citet{Chung2023DPS}.

However, many criticism can be said to this canonical twisting. 
First, in almost all practical cases, the canonical twisting has to be approximated, and the approximation errors are substantial and difficult to control particularly when $N - k$ is large. 
Second, even if we were able to compute the canonical twisting exactly, constructing an effective proposal $M_{k \cond k-1}$, see Remark~\ref{remark:bootstrap-guided}, is hard as well. 
These two problems are especially pronounced when the likelihood $f$ is informative, or $y$ is an outlier observation. 
To build some intuition, consider a heuristic: suppose that we have $J$ samples $\lbrace U^j_0 \rbrace_{j=1}^J \sim q_0$, then, to get weighed samples for $Q_0(\cdot) \propto l^y_0(\cdot) \, q_0(\cdot)$, Equation~\eqref{equ:twisting-canonical} is essentially pushforwarding $\lbrace U^j_0 \rbrace_{j=1}^J$ to $\lbrace U^j_N \rbrace_{j=1}^J$ and then obtain the weights by evaluating $f(y \cond \lbrace U^j_N \rbrace_{j=1}^J)$ which are likely to be degenerate under outlier $y$. 

\section{Construction of twisting sequence with observation path}
\label{sec:corr-fk}
After all, if our primary interest is only the marginal $q_N(\cdot \cond y)$, we are not obliged to constrain $Q_{0:N}$ to follow the SDE representation in Equation~\eqref{equ:cond-sde}.
In this section, we depart from the canonical design, and develop a new construction of the twisting functions $\lbrace l^{v_k}_k \rbrace_{k=0}^N$, especially tailored for informative likelihood or outlier observation. 
Our gist consists in designing $l^{v_k}_k$ in a way that it is an ``appropriate'' likelihood for $q_k$ at $k$, by making a smooth bridging/tempering between $l_0^{v_0}$, the starting likelihood, and $l_N^{v_N}$, the target likelihood. 
To clarify, for instance, if $q_0$ is a standard Gaussian, then we ideally would like $l_0^{v_0} \approx \mathrm{N}(0; 0, I)$ facilitating importance sampling at this step. 
We achieve this by drawing inspiration from~\citet{Corenflos2024FBS}, \citet{Dou2024dps}, and~\citet{Trippe2023diffusion} to construct an auxiliary path of observations correlated with the diffusion model. 
The intermediate likelihood will then reflect this correlation at all steps. 

We define a new process $\lbrace Y_k \rbrace_{k=0}^N$ initialised with the given observation $Y_0 = y \equiv y_0$, and simulate the path $\lbrace Y_k \rbrace_{k=0}^N$ by a noising transition $p_{k \cond k-1}(y_k \cond y_{k-1})$.\footnote{For simplicity, we let the noising process for $y$ be the same as in Equation~\eqref{equ:generative-model} for $x$, but they do not have to be the same in practice.} 
Write its time-reversal as $\lbrace V_k = Y_{N - k} \rbrace_{k=0}^N$. 
We design the twisting function as a time-pushforward of the target likelihood:
\begin{equation}
	l_{N - k}^{v_{N - k}}(x_k) = p_{k}(y_{k} \cond x_k)
	\label{equ:aux-twisted}
\end{equation}
which stands for an intermediate likelihood function in $x_k$ under the observation $Y_k = y_k$. 
Clearly, the terminal constraint $l_N^{v_N}(\cdot) = p_0(y \cond \cdot) = f(y \cond \cdot)$ is satisfied. 
Moreover, it is more suitable compared to the canonical construction in Equation~\eqref{equ:twisting-canonical} for weighting samples at step $k$, as it smoothly interpolates from $p_N(y_N \cond \cdot)$ to $f(y \cond \cdot)$.
This facilitates better importance weighting at intermediate steps.
The remaining blocker is how to compute the likelihood $p_k(y_k \cond x_k)$ which, like the canonical twisting, is generally intractable too.
However, it turns out that if the target likelihood $f$ is linear Gaussian, then we can easily approximate it with a recursion kernel as we describe next. 

\subsection{Construction with Gaussian likelihood models}
\label{sec:gaussian}
Recall from Equation~\eqref{equ:aux-twisted} that our goal is to efficiently approximate the twisting $p_{k}(y_{k} \cond x_k)$ given by
\begin{equation}
	p_k(y_k \cond x_k) = \int p_k(y_k \cond y_0) \, f(y_0 \cond x_0) \, p_0(x_0 \cond x_k) \diff y_0 \diff x_0,
\end{equation}
which is in general intractable due to the denoising $p_0(x_0 \cond x_k)$. 
It seems that the same problem of the canonical twisting in Equation~\eqref{equ:twisting-canonical} persists, as existing approximations to $p_0(x_0 \cond x_k)$ will result in large errors due to the time leap from $k$ to $0$. 
However, it turns out that we can obtain $p_k(y_k \cond x_k)$ based on $p_{k-1}(y_{k-1} \cond x_{k-1})$ for which we can then recursively compute $p_k(y_k \cond x_k)$ starting at the target likelihood $p_0(y_0 \cond x_0) = f(y_0 \cond x_0)$. 
This is shown in the following proposition.

\begin{proposition}
	For any $k=0,1,\ldots$, the intermediate likelihood
	\begin{equation}
		p_k(y_k \cond x_k) = \mathcal{K}^k(f)(y_k, x_k)
		\label{equ:intermediate-likelihood}
	\end{equation}
	is given by $k$-times applications of an operator $\mathcal{K} \colon (\R^c\times\R^d\to\R_+) \to (\R^c\times\R^d\to\R_+)$ on $f$, defined by
	\begin{equation*}
		\mathcal{K}(\psi)(y_k, x_k) \coloneqq \int p_k(y_k \cond y_{k-1}) \, \psi(y_{k-1}, x_{k-1}) \, p_{k-1}(x_{k-1} \cond x_k) \diff y_{k-1} \diff x_{k-1}.
	\end{equation*}
\end{proposition}
\begin{proof}
	We can obtain $p_k(y_k \cond x_k) = \int p_k(y_k, y_{k-1}, x_{k-1} \cond x_k) \diff y_{k-1} \diff x_{k-1} = \int p_k(y_k \cond y_{k-1}, x_{k-1}, x_k) \, p_{k-1}(y_{k-1}, x_{k-1} \cond x_k) \diff y_{k-1} \diff x_{k-1} = \int p_k(y_k \cond y_{k-1}) \times\\ p_{k-1}(y_{k-1} \cond x_{k-1}) \, p_{k-1}(x_{k-1} \cond x_k) \diff y_{k-1} \diff x_{k-1}$ if we know $p_{k-1}(y_{k-1} \cond x_{k-1})$. 
	The result is concluded by iteratively applying the recursion.
\end{proof}
\begin{remark}
	If the observation noising process $p_{k-1}(y_k \cond y_{k-1})$ is $\phi$-stationary: $\int p_{k-1}(y_k \cond y_{k-1}) \diff \phi(y_{k-1}) = \phi(y_k)$, then $\mathcal{K}(\phi)(y, \cdot) = \phi(y)$ is a fixed point.
	The resulting twisting function will be approximately a standard Normal if the noising is an OU process.
\end{remark}

Recall that the canonical twisting in Equation~\eqref{equ:twisting-canonical} requires global approximations to $p_0(x_0 \cond x_k)$ between $0$ and $k$, which is hard. 
In contrast, the interpolating twisting in Equation~\eqref{equ:intermediate-likelihood} only needs local approximations between continuum $k-1$ and $k$, significantly reducing complexity and improving numerical robustness in practice.
We show a particular case when the target likelihood is linear Gaussian for how to approximate $\mathcal{K}^k(f)$ in closed form.

Let the denoising process be defined by
\begin{equation}
	p_{k-1}(x_{k-1} \cond x_k) = \mathrm{N}(x_{k-1}; r_k(x_k), C_k),
\end{equation}
where $r_k\colon \R^d\to\R^d$ is a non-linear denoising function~\citep[e.g., the sampling function that contains the noise predictor in DDPM,][]{Ho2020DDPM}, and $C_k\in\R^{d\times d}$ is a covariance matrix. 
With a slight abuse of notation, this is equivalent to $q_{N - k + 1 \cond N - k}(u_{N - k + 1} \cond u_{N - k})$ in earlier sections but differs only in terms of whether the time flows forward or backward.
Assume that the target likelihood $f(y \cond x) = \mathrm{N}(y; H \, x + b, R)$ is a Gaussian with operator $H\colon \R^d \to \R^c$, bias $b\in\R^c$, and covariance $R\in\R^{c\times c}$. 
Further, choose the auxiliary process $p_{k \cond k-1}(y_k \cond y_{k-1}) = \mathrm{N}(y_k ; A_{k-1} \, y_{k-1}, \Sigma_{k-1})$ for some matrix $A_{k-1} \in\R^{c\times c}$ and covariance $\Sigma_{k-1}\in\R^{c\times c}$.\footnote{This needs not to be the same as the noising process in Equation~\eqref{equ:generative-model}.} 
The following result shows that we can approximate $p_k(y_k \cond x_k)$ by Gaussian with closed-form mean and covariance.

\begin{proposition}[Sequential zeroth-order twisting approximation]
	\label{prop:zero-order-approx}
	At any $k=0,1,\ldots$, define
	\begin{equation*}
		\begin{split}
			\widetilde{\mathcal{K}}(\psi)(y_k, x_k) &\coloneqq \int p_k(y_k \cond y_{k-1}) \, \psi(y_{k-1}, x_{k-1}) \, \mathrm{N}(x_{k-1} ; x_k, C_k) \diff y_{k-1} \diff x_{k-1}, \\
			&\approx \mathcal{K}(\psi)(y_k, x_k).
		\end{split}
	\end{equation*}
	Then $\widetilde{p}_k(y_k \cond x_k) \coloneqq \widetilde{\mathcal{K}}^k(f)(y_k, x_k) = \mathrm{N}(y_k; F_k \, x_k + z_k, \Omega_k)$, where
	\begin{equation*}
		\begin{split}
			F_k &= S_0^k \, H, \\
			z_k &= S_0^k \, b, \\
			\Omega_k &= 
			\begin{cases}
				R, & k=0, \\
				S_0^1 \, (H \, C_1 \, H^\trans + R) \, (S_0^1)^\trans + \Sigma_0, & k=1, \\
				S_0^k \Bigl( H \Bigl(\sum\limits_{i=1}^k C_i\Bigr) H^\trans + R\Bigr) (S_0^k)^\trans + \sum\limits_{i=0}^{k-2} S_{i+1}^k \Sigma_i (S_{i+1}^k)^\trans + \Sigma_{k-1}, & k>1,
			\end{cases}
		\end{split}
	\end{equation*}
	with semigroup $S_m^n \coloneqq \prod_{i=m}^{n-1} A_i$.
\end{proposition}
\begin{proof}
	Suppose that $\widetilde{p}_k(y_k \cond x_k) = \mathrm{N}(y_k; F_k \, x_k + z_k, \Omega_k$, then $\widetilde{p}_{k+1}(y_{k+1} \cond x_{k+1}) = \mathrm{N}(y_{k+1}; F_{k+1} \, x_{k+1} + z_{k+1}, \Omega_{k+1})$ with
	\begin{equation*}
		\begin{split}
			F_{k+1} &= A_k \, F_k, \quad z_{k+1} = A_k \, z_k, \\
			\Omega_{k+1} &= A_k \, (F_k \, C_{k+1} \, F_k^\trans + \Omega_k) \, A_k^\trans + \Sigma_k.
		\end{split}
	\end{equation*}
	Applying the equation above recursively starting at $F_0 = H$, $z_0 = b$, and $\Omega_0 = R$ concludes the result.
\end{proof}

\begin{remark}
	Suppose that $C_k$, $A_k$, and $\Sigma_k$ are constants, and $(A, \Sigma)$ leaves $\mathrm{N}(0, I)$ invariant. 
	Then $\lim_{k\to\infty} \widetilde{p}_k(y_k \cond x_k) = \mathrm{N}(0, I)$ as well, suitable for importance sampling with $\refmeasure \approx \mathrm{N}(0, I)$. 
\end{remark}

From now on, we will use the proposed twisting function
\begin{equation}
	l_k^{v_k}(\cdot) = \widetilde{\mathcal{K}}^{N - k}(f)(y_{N - k}, \cdot)
\end{equation}
in the reverse-time notation. 
Although the twisting is approximate, the marginal $q_N(\cdot \cond y)$ that we are primarily concerned is left un-approximated. 
Furthermore, this twisting function gives us a tractable guided proposal $M_{k \cond k-1}$ mentioned in Remark~\ref{remark:bootstrap-guided} allowing for efficiently sampling:
\begin{equation}
	\begin{split}
		M_{k \cond k-1}(u_k \cond u_{k-1}) &\propto l_k^{v_k}(u_k) \, q_{k \cond k-1}(u_k \cond u_{k-1}) \\
		&= \mathrm{N}\bigl(u_k; \mu(u_{k-1}), C_{N - k} - D_k \, F_{N - k} \, C_{N-k}\bigr), \\
		\mu(u_{k-1}) &=r_{N-k}(u_{k-1}) \! + \! D_k (v_k - F_{N - k} \, r_{N - k}(u_{k-1}) - z_{N - k}), \\
		D_k &= C_{N - k} \, F_{N - k}^\trans \, \bigl( F_{N - k} \, C_{N-k} \, F_{N-k}^\trans + \Omega_{N - k}\bigr)^{-1},
	\end{split}
	\label{equ:cosmc-proposal}
\end{equation}
and as a result the potential function $G_k$ adapts to
\begin{equation}
	G_k(u_k, u_{k-1}) = \frac{l_k^{v_k}(u_k) \, q_{k \cond k-1}(u_k \cond u_{k-1})}{l_{k-1}^{v_{k-1}}(u_{k-1}) \, M_{k \cond k-1}(u_k \cond u_{k-1})}. 
	\label{equ:cosmc-potential}
\end{equation}
We remark that the matrix inversion in Equation~\eqref{equ:cosmc-proposal} is not a computational hurdle for most applications using generative diffusion models: the inversion can be done offline prior to running the SMC sampler, and the dimension does not depend on $d$. 
Even if we want to compute Equations~\eqref{equ:cosmc-proposal} and~\eqref{equ:cosmc-potential} online in $k$, the matrix inversions can be fast solved with one eigendecomposition, when the noising processes' coefficients are scalar. 

The final guided SMC sampler using the proposed twisting in Proposition~\ref{prop:zero-order-approx}, is summarised in Algorithm~\ref{alg:b0-smc}. 
We call it B$^0$SMC reflecting how we construct the twisting with a likelihood bridging, and how we compute it with a zero-th order approximation.

\begin{algorithm2e}[h]
	\SetAlgoLined
	\DontPrintSemicolon
	\KwInputs{Diffusion model for $\pi$, observation $y$, likelihood model $H$, $b$, and $R$, number of samples $J$, auxiliary transition matrices $\lbrace A_k \rbrace_{k=1}^N$ and covariances $\lbrace\Sigma_k\rbrace_{k=1}^N$}
	\KwOutputs{Weighted samples $\lbrace (w_N^j, U_N^j) \rbrace_{j=1}^J \sim \pi(\cdot \cond y)$.}
	Set $y_0 = y$ and sample $y_{1:N} = \lbrace y_1, y_2, \ldots, y_N \rbrace$ based on $p_{k \cond k-1}(y_k \cond y_{k-1})$\;
	Reverse $v_k = y_{N-k}$ for $k=0, 1, \ldots, N$\;
	Compute $\lbrace F_k, z_k, \Omega_k \rbrace_{k=1}^N$ in Proposition~\ref{prop:zero-order-approx}\;
	Using $M_{k\cond k-1}$ and $G_k$ in \eqref{equ:cosmc-proposal} and~\eqref{equ:cosmc-potential}\;
	Draw $J$ samples $\lbrace U_0^j \rbrace_{j=1}^J \sim M_0$\;
	Weight $w_0^j = G_0(U_0^j) \, / \, \sum_{i=1}^J G_0(U_0^i)$ \tcp*{For $j=1,\ldots, J$}
	\For{$k=1,2,\ldots, N$}{%
		Resample $\bigl\lbrace (w_{k-1}, U_{k-1}^j) \bigr\rbrace_{j=1}^J$ if needed \;
		Draw $U_k^j \sim M_{k \cond k-1}(\cdot \cond U_{k-1}^j)$ \tcp*{For $j=1,\ldots, J$}
		Weight $\overline{w}_k^j = w_{k-1}^j \, G_k(U_k^j, U_{k-1}^j)$ \tcp*{For $j=1,\ldots, J$}
		Normalise $w_k^j = \overline{w}_k^j \, / \, \sum_{i=1}^J \overline{w}_k^i$ \tcp*{For $j=1,\ldots, J$}
	}
	\caption{Guided B$^0$SMC for sampling the diffusion posterior $\pi(x \cond y) \propto \mathrm{N}(y; H\, x + b, R) \, \pi(x)$. }
	\label{alg:b0-smc}
\end{algorithm2e}

\section{Experiments}
\label{sec:experiments}
In this section we validate the proposed B$^0$SMC sampler in Algorithm~\ref{alg:b0-smc} on a high-dimensional conditional sampling problem ($d=256$). 
Importantly, we focus on gauging how statistically robust the method is against different levels of outlier observation and number of particles. 
The implementation is performed on an NVIDIA A100 GPU using JAX~\citep{Jax2018github}, and the code is published at \textcolor{magenta}{\url{https://github.com/zgbkdlm/gfk}} for reproducibility.

The target posterior distribution is defined by a Gaussian mixture posterior
\begin{equation}
	\begin{split}
		\pi(x \cond y) &\propto \mathrm{N}(y^\omega ; H \, x, R) \, \pi(x), \quad \pi(x) = \sum_{i=1}^{10} \gamma_i \, \mathrm{N}(x; m_i, \Lambda_i),\\
		y^\omega &= \mathbb{E}_{y\sim \pi(y)}[y] + \omega
	\end{split}
	\label{equ:gm}
\end{equation}
where the mixture weight $\gamma_i$, mean $m_i \in\R^{256}$, and covariance $\Lambda_i\in\R^{256 \times 256}$, and the observation components $H\in\R^{1\times 256}$ and $R$, are all randomly drawn according to a distribution (see Appendix~\ref{appendix:details} for details). 
The observation $y^\omega$ is controlled by an outlier level $\omega$ that moves it away from the most likely position. 
For this model, the posterior distribution is also a Gaussian mixture with tractable means and covariances. 
Importantly, we can also obtain a tractable diffusion model to precisely target this $\pi$, ablating common errors such as from training a neural network or from assuming large enough $N$. 
The noising and denoising models are given by
\begin{equation}
	\begin{split}
		\diff X_t &= -X_t \diff t + \sqrt{2}\diff W_t, \quad X_0 \sim \pi, \\
		\diff U_t &= U_t + 2\grad\log p_t(U_t, T - t) \diff t + \sqrt{2} \diff B_t, \quad U_0 \sim p_T, 
	\end{split}
	\label{equ:test-diffusion}
\end{equation}
where we select $T = 2$ and discretise the SDEs with $N=100$ steps $0=t_0<t_1<\cdots<t_N=T$. 
The noising marginal $p_T$ is also a Gaussian mixture, and initialising the denoising process with it gives $U_T \sim \pi$. 
The details are also shown in Appendix~\ref{appendix:details}.

We compare to one baseline and two state-of-the-art methods for the conditional diffusion sampling. 
The baseline method is diffusion posterior sampling (DPS) developed by~\citet{Chung2023DPS}, a common and practical sampler for diffusion inverse problems. 
The two state-of-the-art methods are TDS by~\cite{Wu2023practical} and MCGDiff by~\citet{Cardoso2024monte}, both are based on Algorithm~\ref{alg:smc} but with different twisting approximations for Equation~\eqref{equ:twisting-canonical}. 
Note that for MCGDiff we use its noiseless version for ablation comparison, therefore we set $R = 10^{-8}$ for our B$^0$SMC when comparing to MCGDiff.
The methods' performance is measured by sliced Wasserstein distance (SWD, with $1$-norm and 1,000 projections) and effective sample size (ESS). 
Unless otherwise mentioned, all experiments are independently repeated for 100 times and we report the statistics of the results. 

\begin{table}[t!]
	\caption{Comparison of sliced Wasserstein distance (SWD) and effective sample size (ESS) at different outlier levels, with the number of particles being 16,384. Here we report the mean, and standard deviation in parenthesis. 
	We see that B$^0$SMC is consistently the best for all outlier levels.}
	\label{tbl:vs-tds}
	\begin{tabular}{@{}lcccccc@{}}
		\toprule
		& \multicolumn{2}{c}{Outlier level $\omega = 0$} & \multicolumn{2}{c}{Outlier level $\omega = 5$} & \multicolumn{2}{c}{Outlier level $\omega = 10$} \\ \cmidrule(l){2-3} \cmidrule(l){4-5} \cmidrule(l){6-7}
		& SWD (\textdownarrow)                & ESS (\textuparrow)          & SWD                 & ESS           & SWD                  & ESS           \\ \midrule
		DPS  & 1.31 (0.89)         & N/A           & 3.04 (1.04)         & N/A           & 4.00 (\textbf{1.32})          & N/A           \\
		TDS & 0.12 (0.14)         & 14440         & 0.62 (1.02)         & 14142         & 0.87 (1.83)          & 13738         \\
		B$^0$SMC & \textbf{0.06} (\textbf{0.03})         & \textbf{15720}         & \textbf{0.25} (\textbf{0.89})         & \textbf{14984}         & \textbf{0.68} (1.83)          & \textbf{14626}         \\ \bottomrule
	\end{tabular}
\end{table}

\begin{table}[t!]
	\caption{Comparison between B$^0$SMC and MCGDiff with different number of particles. We see that B$^0$SMC is consistently the best for all settings in terms of mean and standard deviation (in parenthesis).}
	\label{tbl:vs-mcgdiff}
	\begin{tabular}{@{}lcccccc@{}}
		\toprule
		& \multicolumn{2}{c}{Nr. particles 1024} & \multicolumn{2}{c}{Nr. particles 4096} & \multicolumn{2}{c}{Nr. particles 16384} \\ \cmidrule(l){2-3} \cmidrule(l){4-5} \cmidrule(l){6-7}
		& SWD (\textdownarrow)                & ESS (\textuparrow)          & SWD                 & ESS           & SWD                  & ESS           \\ \midrule
		MCGDiff  & 0.55 (0.38)         & 884           & 0.38 (0.32)         & 3543           & 0.31 (0.40)          & 14201          \\
		B$^0$SMC & \textbf{0.44} (\textbf{0.35})         & \textbf{970}         & \textbf{0.22} (\textbf{0.16})         & \textbf{3883}         & \textbf{0.11} (\textbf{0.07})          & \textbf{15531}         \\ \bottomrule
	\end{tabular}
\end{table}

Table~\ref{tbl:vs-tds} shows the results compared to DPS and TDS. 
A first and crucial observation from the results is that the proposed method B$^0$SMC outperforms DPS and TDS by an order of magnitude.
This conclusion holds consistently over different outlier levels. 
Even when the outlier level $\omega=0$ is none, B$^0$SMC still works significantly better than TDS, possibly thanks to the tractable Markov proposal in Equation~\eqref{equ:cosmc-proposal}. 
The second but side observation, is that TDS works better than DPS. 
This aligns with the motivation of TDS which essentially used DPS as a proposal and then corrected it with an SMC sampler in Algorithm~\ref{alg:smc}. 

Comparison to MCGDiff is shown in Table~\ref{tbl:vs-mcgdiff}. 
We separate this from Table~\ref{tbl:vs-tds} because this is a noiseless observation scenario, and therefore the notion of outlier observation is less defined. 
In this table, we find that B$^0$SMC is significantly better than MCGDiff across different number of particles. 
The standard deviation of B$^0$SMC is also consistently better, implying a more stable sampling process.  

\begin{figure}[t!]
	\centering
	\includegraphics[width=\linewidth]{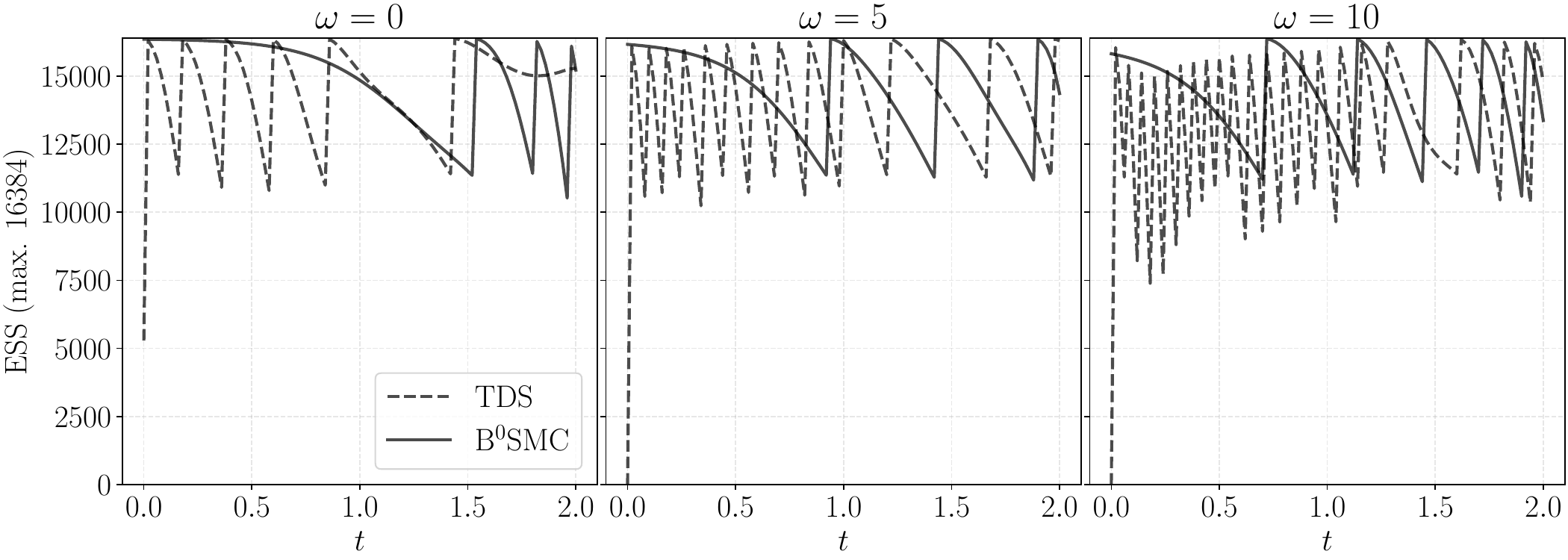}
	\caption{Effective sample sizes (ESS) of TDS and B$^0$SMC at one run. We let the SMC samplers to trigger resampling if the ESS goes below 70\%. We see that B$^0$SMC triggers substantially less resampling. }
	\label{fig:ess}
\end{figure}

\begin{figure}[t!]
	\centering
	\includegraphics[width=.49\linewidth]{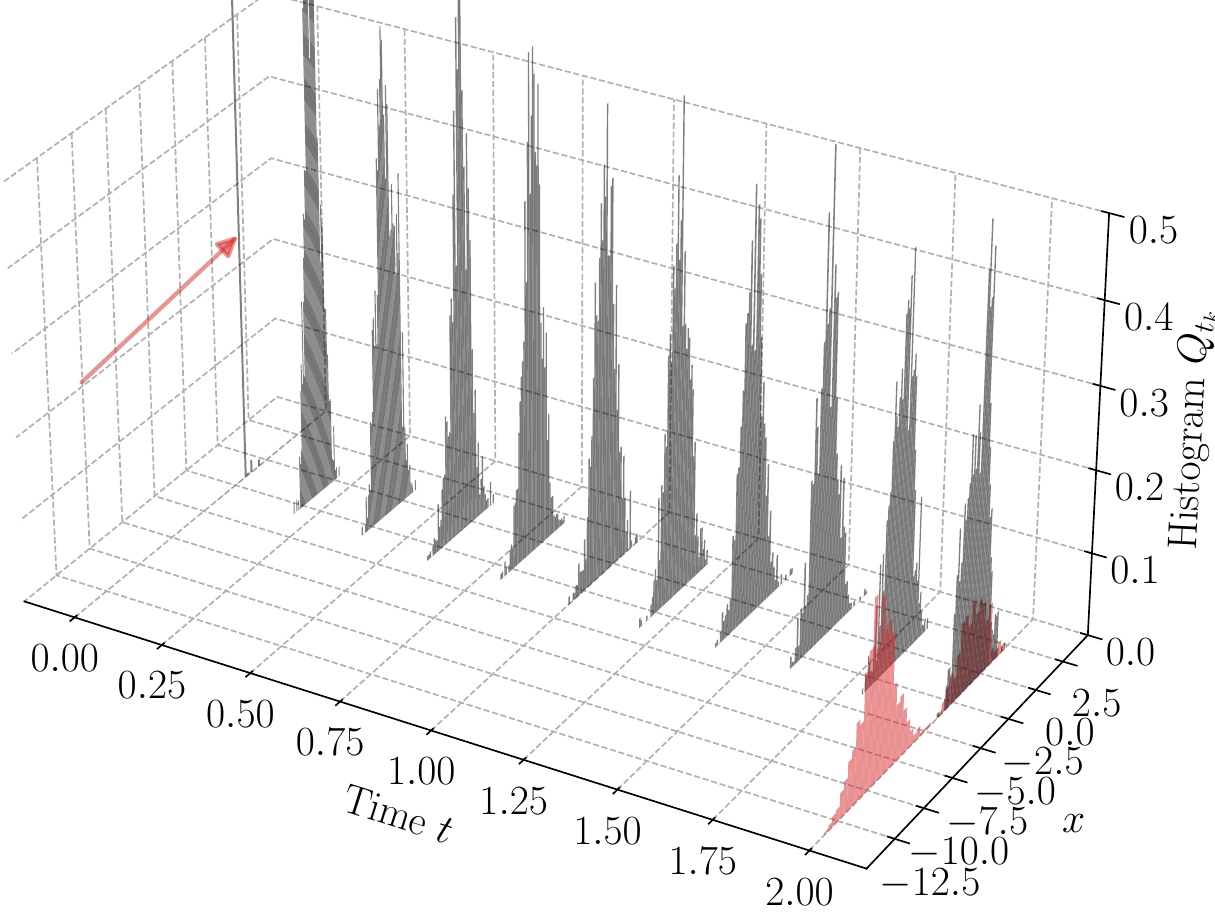}
	\includegraphics[width=.49\linewidth]{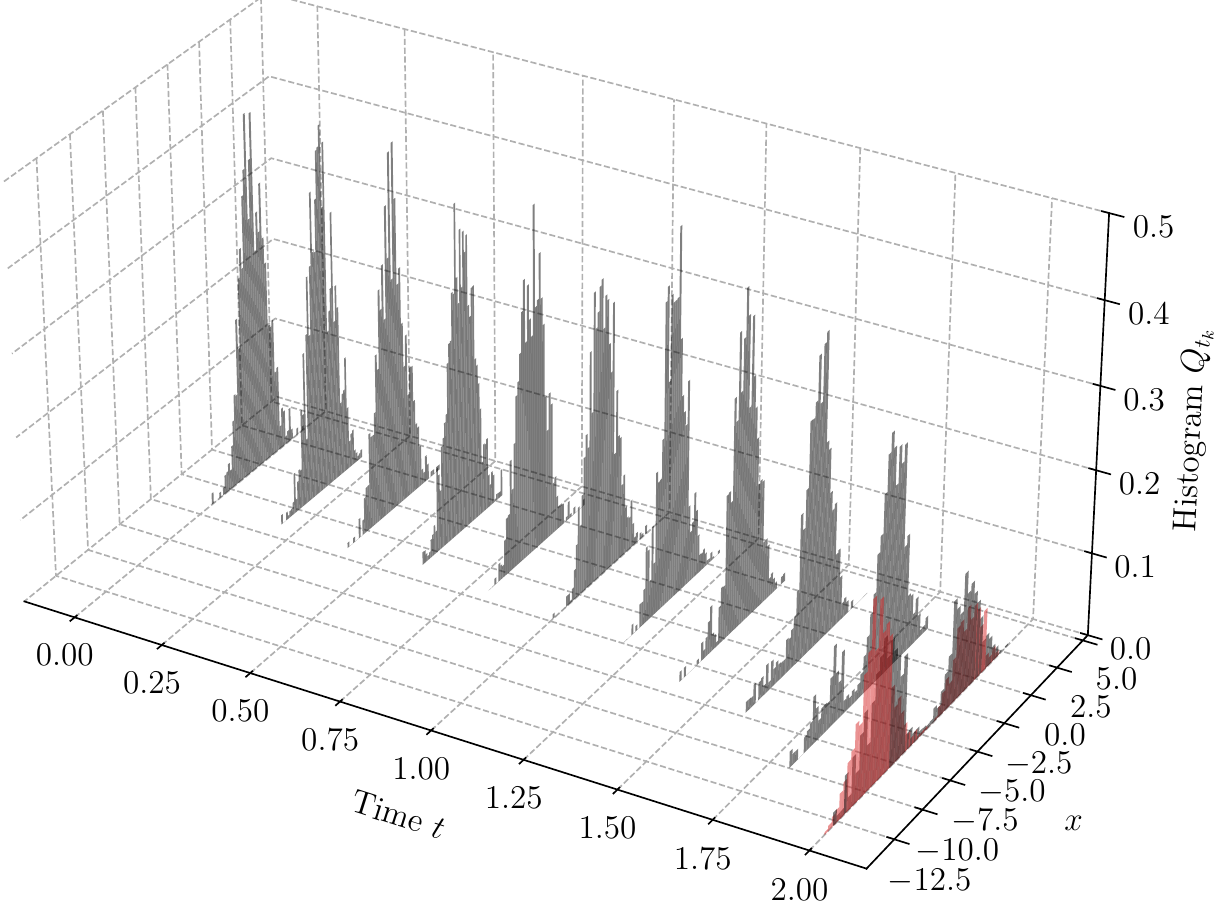}
	\caption{Marginal distribution $Q_k$ of TDS (left) and B$^0$SMC (right) with outlier level $\omega=20$. 
	The histogram in red represents the true posterior distribution. 
	We see that B$^0$SMC correctly recovers the target at $T=2$ whereas TDS does not. Importantly, at $t=0$ pointed by the red arrow, TDS is highly degenerate due to its twisting approximation.}
	\label{fig:marginals}
\end{figure}

For more in-depth comparison, we also plot the ESS and time-evolution of marginal $Q_k$ in Figures~\ref{fig:ess} and~\ref{fig:marginals}. 
The first figure gives us three observations. 
First, we find that in average B$^0$SMC has higher ESS than TDS consistently at different outlier levels. 
Second, B$^0$SMC triggers substantially less resampling, implying that the proposal in Equation~\eqref{equ:cosmc-proposal} is more efficient. 
Third, the initial ESS at $t=0$ of TDS is significantly low. 
This is due to TDS' twisting approximation in Equation~\eqref{equ:twisting-canonical} which is highly erroneous when $T - t$ is large. 
This problem is also reflected in Figure~\ref{fig:marginals}.
We see that TDS at $t=0$ has degenerate samples, and in time, the sampler essentially is moving toward a wrong path, and does not recover the true distribution at $T$. 
On the other hand, the marginal sequence of B$^0$SMC smoothly moves from the reference distribution at $t=0$ and correctly hits the target distribution at $T$.

\section{Conclusion}
\label{sec:conclusion}
In this paper we have developed a new approach for (training-free) conditional sampling of generative diffusion models. 
The method takes any pre-trained diffusion model as prior, and is able to sample the posterior distribution given any pointwise evaluable likelihood function. 
The overall framework is based on twisted Feynman--Kac models and sequential Monte Carlo samplers, where a crucial part of it lies in the construction of the twisting function.
As such, the community, fairly recently, has been working on developing efficient twisting constructions, see, e.g., seminal work by~\citet{Cardoso2024monte, Janati2024DC} and their derivatives. 
Here, we take one step further, and we have constructed a new twisting sequence that provides better statistical performance when the likelihood part poses challenges for posterior sampling. 
Our empirical results show that the proposed method outperforms the peer methods in terms of less posterior sample approximations and higher effective sample size, particularly, when the observation is an outlier. 


\section*{Acknowledgement}
This work was partially supported by 1) the Wallenberg AI, Autonomous Systems and Software Program (WASP) funded by the Knut and Alice Wallenberg Foundation, and 2) the Stiftelsen G.S Magnusons fond (MG2024-0035). 
I thank the Division of Systems and Control at Uppsala University for providing computational resources, and Fredrik Lindsten for discussing this paper. 

I also acknowledge that as a paper on generative diffusion models, real experiments to broaden the impact, such as generating cat images, seem to be missing. 
Verily, I have no excuse for this omission, and defer this to a working future paper.

Lastly but not least, this work is committed to celebrating the 90th birthday of Thomas Kailath for his contributions in control and signal processing.

\bibliography{refs}

\begin{thebibliography}{}

\bibitem[Albergo et~al., 2023]{Albergo2023stochastic}
Albergo, M.~S., Boffi, N.~M., and Vanden-Eijnden, E. (2023).
\newblock Stochastic interpolants: a unifying framework for flows and
  diffusions.
\newblock {\em arXiv preprint arXiv:2303.08797}.

\bibitem[Benton et~al., 2024]{Benton2024}
Benton, J., Shi, Y., De~Bortoli, V., Deligiannidis, G., and Doucet, A. (2024).
\newblock From denoising diffusions to denoising {M}arkov models.
\newblock {\em Journal of the Royal Statistical Society Series B: Statistical
  Methodology}, 86(2):286--301.

\bibitem[Bradbury et~al., 2018]{Jax2018github}
Bradbury, J., Frostig, R., Hawkins, P., Johnson, M.~J., Leary, C., Maclaurin,
  D., Necula, G., Paszke, A., Vander{P}las, J., Wanderman-{M}ilne, S., and
  Zhang, Q. (2018).
\newblock {JAX}: composable transformations of {P}ython+{N}um{P}y programs.

\bibitem[Cardoso et~al., 2024]{Cardoso2024monte}
Cardoso, G., Janati, Y., Corff, S.~L., and Moulines, E. (2024).
\newblock Monte {C}arlo guided denoising diffusion models for {B}ayesian linear
  inverse problems.
\newblock In {\em Proceedings of the the 12th International Conference on
  Learning Representations}.

\bibitem[Chopin and Papaspiliopoulos, 2020]{ChopinBook2020}
Chopin, N. and Papaspiliopoulos, O. (2020).
\newblock {\em An introduction to sequential {M}onte {C}arlo}.
\newblock Springer Series in Statistics. Springer.

\bibitem[Chung et~al., 2023]{Chung2023DPS}
Chung, H., Kim, J., McCann, M.~T., Klasky, M.~L., and Ye, J.~C. (2023).
\newblock Diffusion posterior sampling for general noisy inverse problems.
\newblock In {\em Proceedings of the 11th International Conference on Learning
  Representations}.

\bibitem[Corenflos et~al., 2025]{Corenflos2024FBS}
Corenflos, A., Zhao, Z., S\"{a}rkk\"{a}, S., Sj\"{o}lund, J., and Sch\"{o}n,
  T.~B. (2025).
\newblock Conditioning diffusion models by explicit forward-backward bridging.
\newblock In {\em Proceedings of the 28th International Conference on
  Artificial Intelligence and Statistics}, volume 258, pages 3709--3717. PMLR.

\bibitem[Daras et~al., 2024]{Daras2024survey}
Daras, G., Chung, H., Lai, C.-H., Mitsufuji, Y., Ye, J.~C., Milanfar, P.,
  Dimakis, A.~G., and Delbracio, M. (2024).
\newblock A survey on diffusion models for inverse problems.
\newblock {\em arXiv preprint arXiv:2410.00083}.

\bibitem[De~Bortoli et~al., 2021]{DeBortoli2021diffusion}
De~Bortoli, V., Thornton, J., Heng, J., and Doucet, A. (2021).
\newblock Diffusion {S}chr{\"o}dinger bridge with applications to score-based
  generative modeling.
\newblock In {\em Advances in Neural Information Processing Systems},
  volume~34, pages 17695--17709.

\bibitem[Del~Moral, 2004]{DelMoral2004}
Del~Moral, P. (2004).
\newblock {\em Feynman-Kac formulae: genealogical and interacting particle
  systems with applications}.
\newblock Springer New York.

\bibitem[Del~Moral and Murray, 2015]{DelMoral2015}
Del~Moral, P. and Murray, L.~M. (2015).
\newblock Sequential monte carlo with highly informative observations.
\newblock {\em SIAM/ASA Journal on Uncertainty Quantification}, 3(1):969--997.

\bibitem[Dou and Song, 2024]{Dou2024dps}
Dou, Z. and Song, Y. (2024).
\newblock Diffusion posterior sampling for linear inverse problem solving: a
  filtering perspective.
\newblock In {\em Proceedings of the 12th International Conference on Learning
  Representations}.

\bibitem[Ho et~al., 2020]{Ho2020DDPM}
Ho, J., Jain, A., and Abbeel, P. (2020).
\newblock Denoising diffusion probabilistic models.
\newblock In {\em Advances in Neural Information Processing Systems},
  volume~33, pages 6840--6851. Curran Associates, Inc.

\bibitem[Janati et~al., 2024]{Janati2024DC}
Janati, Y., Moufad, B., Durmus, A., Moulines, E., and Olsson, J. (2024).
\newblock Divide-and-conquer posterior sampling for denoising diffusion priors.
\newblock In {\em Advances in Neural Information Processing Systems},
  volume~37, pages 97408--97444. Curran Associates, Inc.

\bibitem[Kelvinius et~al., 2025]{Kelvinius2025DDSMC}
Kelvinius, F.~E., Zhao, Z., and Lindsten, F. (2025).
\newblock Solving linear-{G}aussian {B}ayesian inverse problems with decoupled
  diffusion sequential {M}onte {C}arlo.
\newblock In {\em Proceedings of the 42nd International Conference on Machine
  Learning (ICML)}.

\bibitem[Luo et~al., 2025]{Luo2024rsta}
Luo, Z., Gustafsson, F.~K., Zhao, Z., Sj\"{o}lund, J., and Sch\"{o}n, T.~B.
  (2025).
\newblock Taming diffusion models for image restoration: a review.
\newblock {\em Philosophical Transactions of the Royal Society A: Mathematical,
  Physical and Engineering Sciences}, 383(2299):20240358.

\bibitem[Shen et~al., 2025]{Shen2025reverse}
Shen, X., Meinshausen, N., and Zhang, T. (2025).
\newblock Reverse {M}arkov learning: Multi-step generative models for complex
  distributions.
\newblock {\em arXiv preprint arXiv:2502.13747}.

\bibitem[Song et~al., 2021]{Song2021scorebased}
Song, Y., Sohl-Dickstein, J., Kingma, D.~P., Kumar, A., Ermon, S., and Poole,
  B. (2021).
\newblock Score-based generative modeling through stochastic differential
  equations.
\newblock In {\em Proceedings of the 9th International Conference on Learning
  Representations}.

\bibitem[Trippe et~al., 2023]{Trippe2023diffusion}
Trippe, B.~L., Yim, J., Tischer, D., Baker, D., Broderick, T., Barzilay, R.,
  and Jaakkola, T.~S. (2023).
\newblock Diffusion probabilistic modeling of protein backbones in {3D} for the
  motif-scaffolding problem.
\newblock In {\em Proceedings of the 11th International Conference on Learning
  Representations}.

\bibitem[Wu et~al., 2023]{Wu2023practical}
Wu, L., Trippe, B.~L., Naesseth, C.~A., Blei, D., and Cunningham, J.~P. (2023).
\newblock Practical and asymptotically exact conditional sampling in diffusion
  models.
\newblock In {\em ICML 2023 Workshop on Structured Probabilistic Inference {\&}
  Generative Modeling}.

\bibitem[Zhao et~al., 2025]{Zhao2024rsta}
Zhao, Z., Luo, Z., Sj\"{o}lund, J., and Sch\"{o}n, T.~B. (2025).
\newblock Conditional sampling within generative diffusion models.
\newblock {\em Philosophical Transactions of the Royal Society A: Mathematical,
  Physical and Engineering Sciences}, 383(2299):20240329.

\end{thebibliography}
\bibliographystyle{apalike}

\appendix
\section{Experiment details}
\label{appendix:details}
The coefficients in Equation~\eqref{equ:gm} are generated as follows.
The mixture weights $\lbrace \gamma_i \rbrace_{i=1}^{10}$ are first independently sampled from a Chi-squared (with degree 1) distribution and are then normalised $\sum_{i=1}^{10}\gamma_i = 1$. 
Also independently for each $i$, the mixture mean $m_i \sim \mathrm{Unif}[-8, 8]^d$ follows a uniform distribution, and the mixture covariance $\Lambda_i = \lambda_i \, \lambda_i^\trans + I_d$, where $\lambda_i\sim \mathrm{Unif}[0, 1]^d$ is also a uniform random variable. 
Let $\widehat{H}\sim \mathrm{N}(0, I_{c\times d})$, $\alpha \sim \mathrm{Unif}[0, 1]^c$, and $\beta \sim \mathrm{Unif}[0, 1]^c$. 
Then we set
\begin{equation*}
	H = U \diagbig{\mathrm{sort}(\alpha) + 10^{-3}} \, V^\trans, 
\end{equation*}
where $U \, S \, V^\trans = \widehat{H}$ is the singular value decomposition of $\widehat{H}$. 
Finally, the observation covariance $R = \beta \, \beta^\trans + \max(\alpha)^2\, I_{c\times c}$. 
When comparing to MCGDiff, the covariance is set to $R=10^{-8}$ for B$^0$SMC.

Recall the test diffusion model in Equation~\eqref{equ:test-diffusion}:
\begin{equation*}
	\begin{split}
		\diff X_t &= a \, X_t \diff t + b\diff W_t, \quad X_0 \sim \pi, \\
		\diff U_t &= a \, U_t + b^2\grad\log p_t(U_t, T - t) \diff t + b \diff B_t, \quad U_0 \sim p_T,
	\end{split}
\end{equation*}
where we chose $a = -1$ and $b = \sqrt{2}$. 
We choose the auxiliary observation process to be the same as this noising process, and as a result, $A_k = \exp(a \, (t_k - t_{k-1}))$ and $\Sigma_k = \frac{b^2}{2 \, a}\bigl(\exp(2 \, a \, (t_k - t_{k-1})) - 1\bigr)$. 
With the Gaussian mixture $\pi$, the reference distribution $p_T = \sum_{i=1}^{10} \gamma^T_i \, \mathrm{N}(m_i^T, \Lambda_i^T)$, where $m_i^T = \exp(a \, T) \, m_i$, and $\Lambda_i^T = E_i \diag{\theta_i^T} \, E_i^\trans$, where $(E_i, \theta_i)$ is the eigendecomposition of $\Lambda_i$ and $\theta_i^T = \exp(2 \, a \, T) \, \theta_i + \frac{b^2}{2 \, a}\bigl(\exp(2 \, a \, T) - 1\bigr)$. 
Using the same routine by replacing $T$ with $t$, the Gaussian mixture marginal $p_t$ is available in closed form, so does its score.
We then use an Euler--Maruyama discretisation at 100 steps $0=t_0 < t_1 < \cdots < t_N = T$ to simulate the denoising process above. 

For all SMC samplers we compare here, we use stratified resampling when the ESS becomes lower than 70\%. 
The number of posterior samples is equal to the number of particles, unlike~\citet{Cardoso2024monte, Kelvinius2025DDSMC} who formed a SMC chain of $J$ particles to generate $M > J$ samples. 
If wish to do so, the SMC sampler should be replaced with a conditional SMC sampler~\citep{Corenflos2024FBS}, otherwise biased.

\begin{figure}[t!]
	\centering
	\includegraphics[width=.95\linewidth]{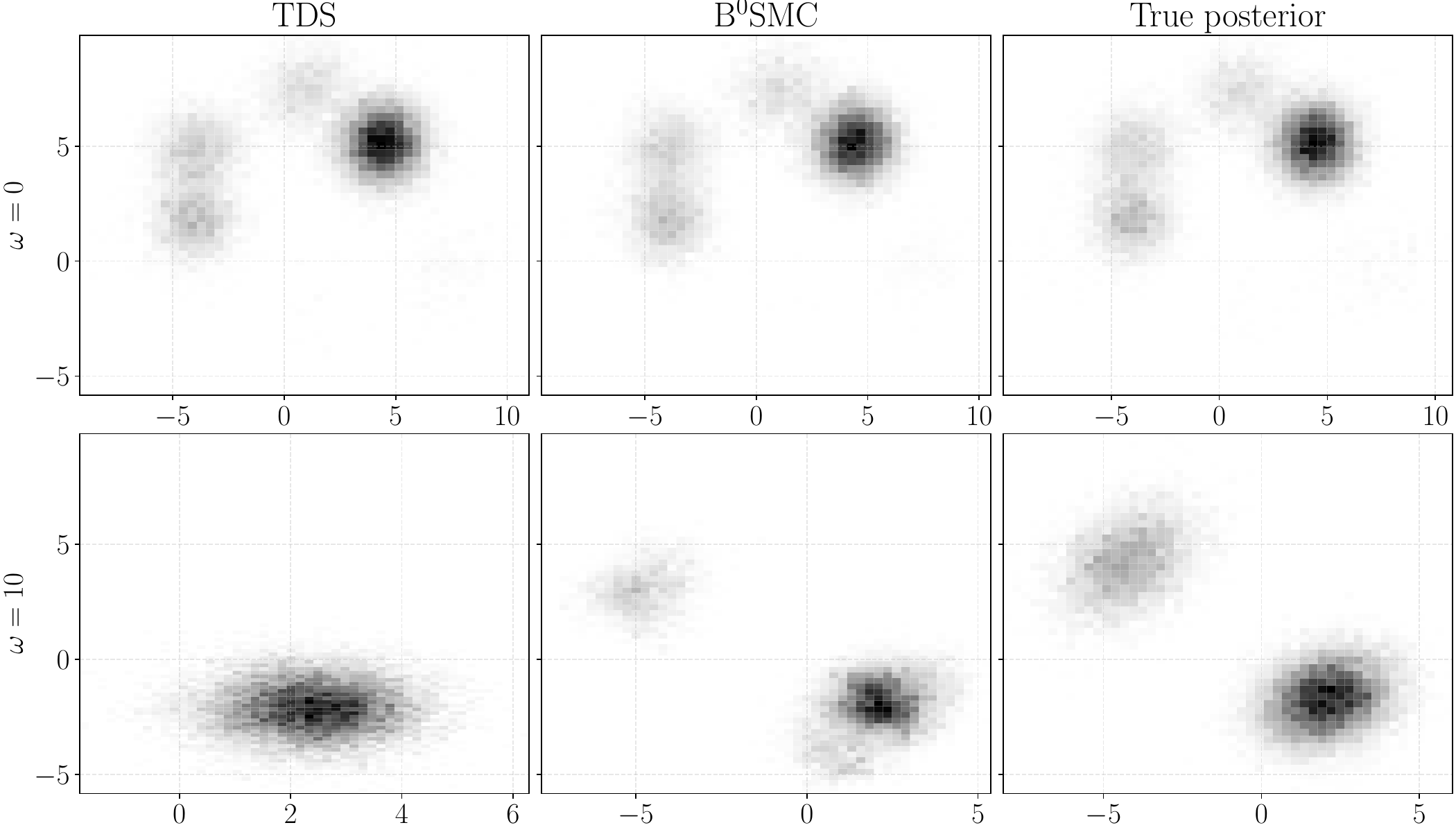}
	\caption{Histogram (2D slice) of the posterior samples drawn by TDS and B$^0$SMC at one run. We see that B$^0$SMC captures the true posterior distribution better than TDS when $\omega$ is high.}
\end{figure}

\end{document}